\documentclass[sigconf]{acmart}
\AtBeginDocument{
  }

\copyrightyear{2026}
\acmYear{2026}
\setcopyright{cc}
\setcctype{by}
\acmConference[KDD '26]{Proceedings of the 32nd ACM SIGKDD Conference on Knowledge Discovery and Data Mining V.2}{August 09--13, 2026}{Jeju Island, Republic of Korea}
\acmBooktitle{Proceedings of the 32nd ACM SIGKDD Conference on Knowledge Discovery and Data Mining V.2 (KDD '26), August 09--13, 2026, Jeju Island, Republic of Korea}
\acmDOI{10.1145/3770855.3818143}
\acmISBN{979-8-4007-2259-2/2026/08}

\usepackage{algorithm}
\usepackage{algpseudocode}
\usepackage{multirow}
\usepackage{booktabs}
\usepackage{bbding}
\usepackage{pifont}
\usepackage{utfsym}
\usepackage{amsmath}
\usepackage{makecell}
\usepackage{array}
\usepackage{amsmath}
\usepackage[inline]{enumitem}
\usepackage{graphicx}
\usepackage{adjustbox}
\usepackage{csquotes}
\usepackage{newfloat}
\usepackage{listings}
\usepackage{bm}
\usepackage{graphicx}
\usepackage{subcaption}
\usepackage{makecell}
\usepackage{threeparttable}
\usepackage{hyperref}
\usepackage{colortbl}
\usepackage{amsthm}

\usepackage{diagbox}
\usepackage{caption}
\usepackage{subcaption}
\usepackage[linesnumbered,ruled,vlined,algo2e]{algorithm2e}

\newtheorem{proposition}{Proposition}

\theoremstyle{definition}

\newtheorem{assumption}{Assumption}

\usepackage{xcolor}
\usepackage[dvipsnames]{xcolor}

\begin{document}

\title{FlowTime: Towards Continuous Generative Watch Time Prediction via Flow-based Personalized Priors}

\author{Hongxu Ma}
\authornote{Both authors contributed equally to this research.}
\authornote{Work done during the internship at Kuaishou Technology.}
\affiliation{
  \institution{Fudan University}
  \city{Shanghai}
  \country{China}}
\email{hxma24@m.fudan.edu.cn}

\author{Han Zhou}
\authornotemark[1]
\affiliation{
  \institution{Shanghai University of Finance and Economics}
  \city{Shanghai}
  \country{China}}
\email{zhouhan@stu.sufe.edu.cn}

\author{Chenghou Jin}
\authornotemark[1]
\affiliation{
  \institution{Fudan University}
  \city{Shanghai}
  \country{China}}
\email{jinch24@m.fudan.edu.cn}

\author{Jie Zhang}
\affiliation{
  \institution{Kuaishou Technology}
  \city{Beijing}
  \country{China}}
\email{zhangjie39@kuaishou.com}

\author{Xiaoyu Yang}
\affiliation{
  \institution{Kuaishou Technology}
  \city{Beijing}
  \country{China}}
\email{yangxiaoyu@kuaishou.com}

\author{Chunjie Chen}
\affiliation{
  \institution{Kuaishou Technology}
  \city{Beijing}
  \country{China}}
\email{chencj517@gmail.com}

\author{Jihong Guan}
\affiliation{
  \institution{Tongji University}
  \city{Shanghai}
  \country{China}}
\email{jhguan@tongji.edu.cn}

\author{Shuigeng Zhou}
\authornote{Corresponding author.}
\affiliation{
  \institution{Fudan University}
  \city{Shanghai}
  \country{China}}
\email{sgzhou@fudan.edu.cn}
\renewcommand{\shortauthors}{Hongxu Ma et al.}

\begin{abstract}
Watch time has emerged as a pivotal metric for optimizing deep user engagement in short-video recommender systems.
However, current methods of watch time prediction (WTP)  suffer from inherent paradigm-specific limitations.
\textit{Direct Regression} faces mean-collapse due to unimodal Gaussian assumptions, while \textit{Ordinal Regression} is hampered by quantization errors from rigid discretization. 
Similarly, \textit{Discrete Generative Regression} struggles with high inference latency and heuristic vocabulary design.
Beyond these specific flaws, a shared deficiency is the inability to capture the intrinsic multimodality and heterogeneity of \textit{User-Item Interaction Patterns}.

To address these challenges, we first revisit the WTP problem from a causal perspective, and identify these user-specific patterns as structural confounders that modulate watch time outcomes, where identical interests manifest as distinct watch time outcomes conditioned on diverse user habits. Then, 
we formally propose a new (or the fourth) paradigm --- \textbf{Continuous Generative Regression}, and introduce \textbf{FlowTime}, a novel method utilizing a One-step Generative Variational Autoencoder.
FlowTime effectively circumvents the latency of iterative denoising while maintaining the expressivity of continuous latent spaces.
Furthermore, we design a \textit{Flow-based Personalized Prior} that leverages NFs to warp a standard Gaussian prior into a complex, history-conditioned manifold, thereby enabling the adaptive modeling of multimodal interaction patterns.
Finally, we build \textit{TimeRec}, the first open-source WTP Library, alongside a novel personalization metric to establish a rigorous benchmarking standard. 
Extensive offline experiments and online A/B tests demonstrate FlowTime's significant superiority over SOTA methods. 
Our code is available at 
\href{https://github.com/snailma0229/TimeRec.git}{\textcolor{Mahogany}{https://github.com/snailma0229/TimeRec.git}}.

\end{abstract}


\maketitle
\section{Introduction}
\label{sec:intro}

\begin{figure}[t]
    \centering
    \vspace{1em}
    \includegraphics[scale=0.55]{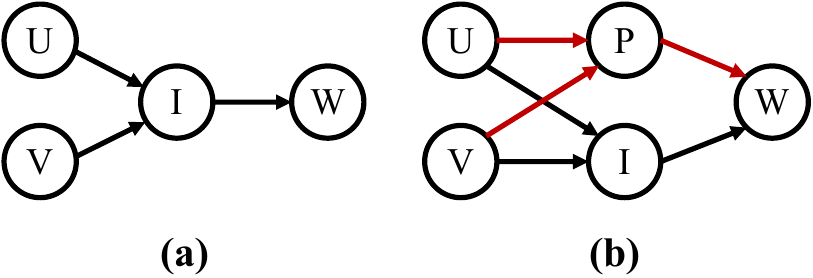}
    \caption{Causal graphs for WTP. Nodes: $U$-User, $V$-Video, $I$-Interest, $P$-Interaction Pattern, and $W$-Watch Time. 
    (a) The traditional view.
    (b) We identify interaction patterns ($P$) as structural confounders (red lines) to modulate outcomes.}
    \label{fig:causal}
    \vspace{-10pt}
\end{figure}

As short-video platforms increasingly dominate global information consumption~\cite{covington2016deep, davidson2010youtube, liu2019user, liu2021concept}, the objective of recommender systems is shifting from traditional Click-Through Rate (CTR) to deeper engagement-centric metrics~\cite{gao2022kuairand,gong2022real}, epitomized by \textit{Watch Time}.
Serving as a pivotal proxy for user satisfaction and immersion, accurate \textit{watch time prediction} (WTP) is critical for optimizing traffic allocation, enhancing user retention and further driving revenue growth~\cite{covington2016deep, wu2018beyond, yi2014beyond}.

WTP techniques evolves through three paradigms, each constrained by its specific limitations: 
(1) \textit{Direct Regression} using MSE-based optimization implicitly assumes a unimodal Gaussian prior, leading to mean-collapse on complex distributions --- a deficiency that remains unrectified despite causal debiasing~\cite{cwm,d2co,d2q,jin2026invariant}.
(2) \textit{Ordinal Regression} discretizes time for interval-wise classification~\cite{tpm, cread, PTPM, ma2026gor} but suffers from rigid quantization granularity and the erroneous assumption of conditional independence among intervals.
(3) \textit{Discrete Generative Regression} adapts language modeling to generate quantized tokens~\cite{ma2024generative}, yet is hampered by heuristic vocabulary design and high latency from autoregressive decoding.

Apart from these paradigm-specific limitations, a shared challenge lies in the intrinsic multimodality and heterogeneity of watch time distributions~\cite{egmn}.
From a causal perspective, traditional WTP techniques simply assume that WTP is solely determined by user interest as shown in Fig.~\ref{fig:causal}(a). Actually, WTP is not merely a function of interest, but influenced by the interplay between interest and \textit{User-Item Interaction Patterns}, which act as structural confounders in Fig.~\ref{fig:causal}(b). For example, as shown in Fig.~\ref{fig:pattern} (a), given the same interest, an ``aggressive'' user exhibits a bimodal distribution (immediate skip \textit{vs}. completion), whereas a ``passive'' user generates a unimodal distribution (hovering). Similarly, Fig.~\ref{fig:pattern} (b) illustrates analogous variations driven by distinct item characteristics.
This implies that identical interests map to distinct distribution topologies conditioned on user-item patterns, ultimately leading to divergent watch time outcomes. However, prior works largely neglect such pattern-driven distribution shifts, failing to capture personalized multimodal distributions and limiting predictive fidelity.

To address the limitations above, inspired by the success of generative models~\cite{ddpm,flowmatching,vae} in modeling intricate data manifolds, we formally propose the \textbf{Continuous Generative Regression Paradigm}, which allows for directly fitting complex, multimodal watch time distributions. However, adapting this paradigm to watch time modeling in recommender systems faces two critical hurdles:

\begin{enumerate}[label=\arabic*., leftmargin=13pt]
  \item \textit{Inference Latency}: The iterative denoising mechanism of mainstream generative models, such as diffusion models~\cite{ddpm}, conflicts with the stringent latency constraints of online systems.
  \item \textit{Prior Mismatch}: The standard Gaussian prior underlying vanilla generative models fails to capture the highly heterogeneous and personalized user-item behavioral patterns.
\end{enumerate}

In response, we propose \textbf{FlowTime}, a novel continuous generative framework. For latency, we adopt a One-step Variational Autoencoder (VAE) architecture, bypassing costly iterative processes to ensure real-time inference efficiency~\cite{lan2025acam,lan2026reco} while learning continuous latent space to maintain expressive power~\cite{zhu2025ett}. To address prior mismatch, we design a \textit{Flow-based Personalized Prior}. 
Conditioned on historical user-item interaction patterns, we leverage NFs to learn bijective transformations that warp the generic Gaussian latent space into a pattern-specific prior manifold.
This personalized topology enables adaptively capturing diverse and complex preference distributions, thereby establishing more accurate ordinal relationships for precise WTP.
Extensive offline and online experiments demonstrate the superiority of our approach across standard accuracy and ranking metrics, as well as our proposed novel metric to quantify personalization fidelity.

Furthermore, to address widespread inconsistencies in data preprocessing and evaluation protocols, we construct the field's first open-source Watch-Time Prediction Library---\textit{TimeRec} . By standardizing evaluation protocols and integrating reproduced SOTA methods, TimeRec provides a rigorous baseline for fair and consistent benchmarking~\cite{wu2026knowme}.

\begin{figure}[t]
    \centering
    \includegraphics[scale=0.62]{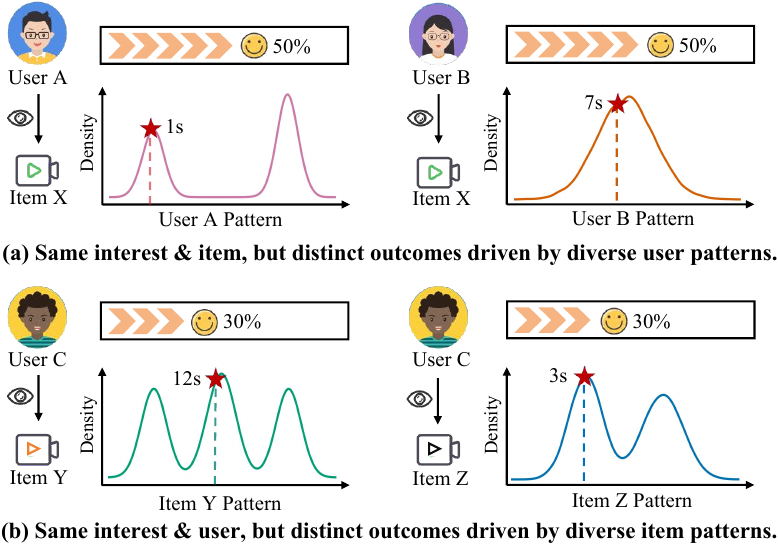}
    \vspace{-1em}
    \caption{Impact of interaction patterns on watch time. 
    Predictions for identical interest levels diverge based on (a) user habits and (b) item characteristics, highlighting the necessity of pattern-aware modeling.}
    \label{fig:pattern}
    \vspace{-1em}
\end{figure}

The main contributions are summarized as follows:
\begin{itemize}[leftmargin=12pt]
  \item We formally propose \textit{Continuous Generative Regression} as a new paradigm for WTP. From a causal perspective, we identify the structural modulation of user-item patterns on WTP, and provide the theoretical superiority of this paradigm in overcoming the intrinsic limitations of traditional regression and discretization paradigms.
  \item We introduce FlowTime, the first framework for the new paradigm. By integrating One-step Generation with a Flow-based Personalized Prior, it reshapes the prior space into pattern-driven manifolds, achieving precise and adaptive modeling of heterogeneous patterns without sacrificing inference efficiency.
  \item Extensive offline experiments and online A/B tests demonstrate that FlowTime significantly outperforms existing state-of-the-art methods. We also provide in-depth analysis of key factors and underlying mechanisms driving the performance gains..
  \item Last but not least, we build the first open-source WTP Library \textit{TimeRec}. Integrating our novel personalization metric, this toolkit provides a rigorous baseline for fair and consistent benchmarking.
\end{itemize}

\section{Related Work}
\label{sec:relat}
\subsection{Watch Time Prediction}
WTP aims to estimate the video watch time based on the user's profile and video characteristics~\cite{zhang2025multi,ma2025fine,ma2025ms}. Existing techniques  fall into three paradigms: (1) \textit{Direct Regression} treats WTP as a point-wise regression task, typically optimized via MSE and often augmented with causal debiasing~\cite{d2co, d2q, cwm,jin2026invariant}. However, the underlying Gaussian assumption inevitably leads to mean-collapse. 
While methods like EGMN~\cite{egmn} attempt to fit mixtures of pre-defined distributions, their performance remains bounded by restrictive parametric assumptions. (2) \textit{Ordinal Regression}~\cite{tpm,cread,swat,PTPM,ma2026gor} transforms the continuous regression problem into a classification task by discretizing time into intervals, suffering from the conditional independence assumption between intervals and rigid predictions due to fixed quantization granularity. (3) \textit{Discrete Generative}~\cite{ma2024generative} adapts the language modeling to reformulate WTP as a sequence generation task by quantizing continuous time into discrete tokens for autoregressive prediction, which is sensitive to heuristic vocabulary design and incurs prohibitive inference latency.
Moreover, these methods collectively fail to capture the intrinsic multimodality and heterogeneity of user-item interaction patterns --- a core challenge that this paper tries to address by proposing a new  \textit{Continuous Generative Paradigm} and the FlowTime framework.

\subsection{Generative Recommendation}
Generative recommendation aims to model the underlying distribution of user-item interactions to directly generate personalized results. 
Early solutions predominantly focused on discrete sequence modeling, utilizing RNNs~\cite{gru4rec} or Transformers~\cite{bert4rec,sasrec} for autoregressive next-item prediction.
GR~\cite{ma2024generative} pioneered the application of this paradigm to WTP by quantizing continuous time into discrete tokens for stepwise autoregressive decoding.
Recently, continuous generative models like VAEs~\cite{vae,kingma2019introduction} and diffusion models~\cite{ddpm, flowmatching} have gained traction for their superior density estimation capabilities, widely employed to generate high-fidelity item or user representations~\cite{diffurec, dreamrec, dimerec}. 
However, these approaches have focused almost exclusively on item generation or representation augmentation, and their application to WTP remains unexplored.

\subsection{Normalizing Flows Modeling}
Normalizing Flows (NFs) were introduced to enhance variational inference through sequences of invertible and differentiable transformations, enabling expressive density modeling with exact likelihood computation~\cite{rezende2015variational,papamakarios2021normalizing}.
As a general framework for high-dimensional density estimation, NFs have driven significant progress in diverse areas, including architectural innovations like RealNVP~\cite{dinh2017density}, continuous probability path modeling via flow matching~\cite{flowmatching}, and representation learning~\cite{papamakarios2021normalizing}.
While NFs have recently emerged in recommendation systems for modeling user preferences~\cite{liang2018variational,liu2025flow}, their potential for WTP remains unexplored.
By leveraging their capability to model complex latent manifolds, we first attempt to harness NFs for WTP and introduce a continuous generative paradigm tailored to capture the heterogeneous interaction patterns.

\section{Existing Modeling Paradigms Revisited}

\subsection{Mean Collapse in Conventional Regression}

A common approach for continuous value prediction is point-wise regression with mean squared error (MSE),
\begin{equation}
\min_f \; \mathbb{E}_{\mathbf{x},y}\big[\|y - f(\mathbf{x})\|^2\big].
\end{equation}
The optimal solution of this objective is given by the conditional expectation:
\begin{equation}
f^*(\mathbf{x}) = \mathbb{E}_{P_{\text{data}}}[y \mid \mathbf{x}],
\end{equation}
which follows from the first-order optimality condition of the MSE objective.
Despite being optimal in expectation, the MSE-optimal predictor may collapse to an average that lies in a low-density region when the conditional distribution is multi-modal.

\begin{proposition}[Mean Collapse Effect]
\label{prop:mean-collapse}
Let $\Delta=\min_{i\neq j}|\mu_i-\mu_j|$ denote the minimum separation between any two modes.
If the modes are well separated, i.e., $\Delta/\sigma \to \infty$, then the likelihood of the regression solution under the true conditional distribution vanishes:
\begin{equation}
\lim_{\Delta / \sigma \to \infty}
P_{\text{data}}\!\left(f^*(\mathbf{x}) \mid \mathbf{x}\right)
=
0.
\end{equation}
\end{proposition}

It shows that point-wise regression collapses to a low-density prediction that fails to reflect the underlying ordinal structure. The formal proof is provided in Appendix~\ref{app:regtheory}.

\subsection{Limitations of Ordinal Regression}

Ordinal regression methods (e.g., CREAD~\cite{cread}, SWaT~\cite{swat}) discretize the continuous watch time $y$ into a sequence of ordered binary decisions using predefined thresholds
$c_1 < \cdots < c_M$, where $\mathbf{B}^m = \mathbb{I}(y > c_m), m=1,\dots,M $.
This transformation encodes $y$ as a length-$M$ binary vector $\mathbf{B} = (\mathbf{B}^1,\dots,\mathbf{B}^M)$, from which the conditional expectation is approximated by discretization:
\begin{equation}
\mathbb{E}(y \mid \mathbf{x})
\approx
\sum_{m=1}^{M} P(y > c_m \mid \mathbf{x}) \, (c_m - c_{m-1}),
\label{eq:ordinal_expectation}
\end{equation}
with $c_0=0$.
This reformulation reduces continuous regression to multiple binary classification problems, but introduces structural bias due to quantization and the implicit assumption of conditional independence across interval decisions ~\cite{ma2024generative}.

\begin{proposition}[Dependency Error in Discretized Modeling]
\label{pro:disError}
Let $P_{\text{data}}(\mathbf{B} \mid \mathbf{x})$ denote the true joint distribution over interval decisions.
A naive discretization model assumes independence,
\begin{equation}
    P_{\text{naive}}(\mathbf{B} \mid \mathbf{x}) = \prod_{m=1}^{M} P(\mathbf{B}^m \mid \mathbf{x}).
\end{equation}

The modeling error decomposes as
\begin{equation}
\begin{aligned}
D_{\mathrm{KL}}
\!\left(
P_{\text{data}} \,\|\, P_{\text{naive}}
\right)
&=\sum_{m=1}^{M} \mathbb{E}_{\mathbf{B}_i^{<m}}  \left[ D_{KL}^{(m)} \right],
\end{aligned}
\end{equation}
 where $D_{KL}^{(m)}$ denotes the KullbackLeibler (KL) divergence between the conditional distribution of the $m$-th interval decision given its history and the marginal distribution that ignores such dependency.
\end{proposition}
The proof is provided in Appendix~\ref{app:creadtheory}.

\subsection{Limitations of Discrete Generative Modeling}
Autoregressive (AR) models~\cite{ma2024generative} mitigate the dependency issue by adopting a generative formulation over a sequence of \emph{value tokens}, rather than independent bucket decisions.
It represents the continuous target $y$ by a sequence $\mathbf{s}=(s^{1},\dots,s^{T})$ where $t$ indexes the generation step and $T$ denotes the sequence length. The conditional distribution is modeled via the AR factorization
$P(\mathbf{s}\mid \mathbf{x})=\prod_{t=1}^{T}P(s^{t}\mid \mathbf{x}, s^{<t})$.
With a numeric decoding map $\phi(\cdot)$, the regression target and prediction are given by $y=\sum_{t=1}^{T}\phi(s^{t})$ and $\hat{y}=\sum_{t=1}^{T}\phi(\hat{s}^{t})$,
where $s^{t}$ and $\hat{s}^{t}$ denote the ground-truth and predicted value tokens at step $t$, respectively.
Accordingly, we define the step-wise regression error as
$\Delta_t \triangleq \phi(\hat{s}^{t}) - \phi(s^{t})$.

\begin{proposition}[Error Decompopstion in Tokenized Autoregressive Regression]
\label{pro:ar_error}

Assume that $\phi(s^{t}), \phi(\hat{s}^{t}) \in [w_{\min}, w_{\max}]$,
where $[w_{\min}, w_{\max}]$ denotes the numeric range induced by the token vocabulary, and that the step-wise bias satisfies
$|\mathbb{E}[\Delta_t]| \le B$ for all $t$. Then, the expected squared error of the AR prediction
$\hat{y} \triangleq \sum_{t=1}^{T} \phi(\hat{s}^{t})$
is upper bounded by
\begin{equation}
\mathbb{E}\!\left[(\hat{y}-y)^2\right]
\;\le\;
T^{2} B^{2}
+
T^{2}\frac{(w_{\max}-w_{\min})^{2}}{4}.
\end{equation}
\end{proposition}
Since the bias term $B$, sequence length $T$, and value range $(w_{\max}-w_{\min})$ all depend on the chosen vocabulary, heuristic discretization imposes an intrinsic limitation on the accuracy of AR models. A formal proof and further analysis is provided in Appendix~\ref{app:arlim}.

\begin{figure*}[t]
    \centering
    \includegraphics[width=0.95\linewidth]{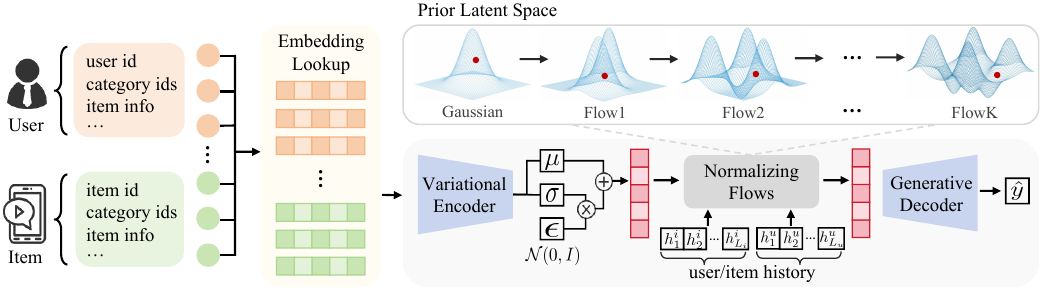}
    \caption{The overall architecture of FlowTime. The framework employs a Variational Encoder to map inputs into a stochastic latent space. Crucially, we introduce a Flow-based Personalized Prior, which leverages Normalizing Flow conditioned on user/item interaction histories to progressively warp a standard Gaussian into a complex, pattern-specific latent manifold, enabling the Generative Decoder to yield precise watch time predictions.}
    \label{fig:framework}
\end{figure*}

\section{Method}
\label{sec:method}

\subsection{Problem Formulation}
Let $\mathcal{D}=\{(\mathbf{u}_i, \mathbf{v}_i, y_i)\}_{i=1}^{N}$ denote a dataset of $N$ samples, where $\mathbf{u}_i, \mathbf{v}_i \in \mathbb{R}^{d}$ denote the user and video feature vectors, respectively, where $d$ is the feature dimension. $y_i \in \mathbb{R}^+$ denotes the observed watch time. Each user and item has their historical watch time sequences, denoted as $\mathcal{H}^u_i = \{h^u_1, \dots, h^u_{L_u}\}$ and $\mathcal{H}^v_i = \{h^v_1, \dots, h^v_{L_v}\}$, respectively. Here, $L_u$ and $L_v$ denote the truncated lengths of the historical sequences.
WTP aims to learn a deterministic mapping $f: \mathbf{x} \to \mathbb{R}^+$, where $\mathbf{x}=(\mathbf{u},\mathbf{v})$, to minimize the discrepancy between the predicted value $\hat{y}$ and ground truth $y$.


Departing from the traditional point estimation paradigm, we reformulate WTP as a continuous generative modeling problem. Our goal is to learn the true underlying conditional probability density $p_{\text{data}}(y \mid \mathbf{x})$, where $y$ is treated as a continuous random variable. This objective is achieved by maximizing the marginal likelihood of the observed watch time $y$ conditioned on $\mathbf{x}$:
\begin{equation}
p_\theta(y \mid \mathbf{x})
=
\int_{\mathcal{Z}}
p_\theta(y \mid z, \mathbf{x}) \,
p_\psi(z \mid \mathbf{x})
\, dz,
\end{equation}
where $z \in \mathbb{R}^{d}$ denotes a latent variable, $\theta$ parameterizes the conditional likelihood (decoder), and $\psi$ parameterizes the conditional prior over the latent space. In essence, this formulation allows the framework to model the multimodality and heterogeneity prevalent in real-world continuous watch time distributions.

\subsection{Overall Architecture of FlowTime}
As illustrated in Fig.~\ref{fig:framework}, FlowTime is built upon a VAE backbone, designed for efficient one-step continuous generative regression.
The framework comprises two principal components: a Probabilistic Encoder that approximates the posterior distribution of the latent variables given the target and user/video features, and a Generative Decoder that reconstructs the watch time distribution from the latent space.

\subsubsection{Probabilistic Encoder}
The encoder approximates the intractable true posterior of the latent variable by conditioning the prior $p_\psi(\mathbf{z} \mid \mathbf{x})$ on the observed watch time. 
Formally, given the static features $\mathbf{x}_i = (\mathbf{u}_i, \mathbf{v}_i)$ and the target $y_i$, the encoder models the approximate posterior distribution $q_{\phi}(\mathbf{z} \mid y_i, \mathbf{x}_i)$ as a multivariate Gaussian:
\begin{equation}
q_{\phi}(\mathbf{z} \mid y_i, \mathbf{x}_i) = \mathcal{N}(\mathbf{z}; \mathbf{\mu}_\phi(y_i, \mathbf{x}_i), \mathbf{\sigma}^2_\phi(y_i, \mathbf{x}_i)\mathbf{I}),
\end{equation}
where $\mathbf{\mu}_\phi(\cdot)$ and $\mathbf{\sigma}_\phi(\cdot)$ are Multi-Layer Perceptrons (MLPs) that predict the mean and diagonal covariance of the latent variable $\mathbf{z} \in \mathbb{R}^d$. To enable backpropagation through the stochastic sampling process, we employ the reparameterization trick~\cite{kingma2013auto}. Specifically, the latent variable $\mathbf{z}$ is sampled as:
\begin{equation}
\mathbf{z} = \mathbf{\mu}_\phi(y_i, \mathbf{x}_i) + \mathbf{\sigma}_\phi(y_i, \mathbf{x}_i) \odot \mathbf{\epsilon}, \quad \text{where} \quad \mathbf{\epsilon} \sim \mathcal{N}(\mathbf{0}, \mathbf{I}),
\end{equation}
where $\odot$ denotes the element-wise product. This transformation decouples the stochastic noise $\mathbf{\epsilon}$ from the network parameters, allowing for end-to-end optimization.

\subsubsection{Generative Decoder}
The decoder functions as the generative network, aiming to reconstruct the target watch time $y_i$ from the sampled latent variable $\mathbf{z}$, conditioned on the static features $\mathbf{x}_i$. We model the likelihood distribution $p_{\theta}(y_i \mid \mathbf{z}, \mathbf{x}_i)$ as:
\begin{equation}
p_{\theta}(y_i \mid \mathbf{z}, \mathbf{x}_i) = \mathcal{N}(y_i; \mathbf{\mu}_\theta(\mathbf{z}, \mathbf{x}_i), \mathbf{\sigma}^2_\theta(\mathbf{z}, \mathbf{x}_i)),
\end{equation}
where $\mathbf{\mu}_\theta$ and $\mathbf{\sigma}_\theta$ are MLP-based networks mapping the latent feature space back to the target space. 
Following the standard VAE formulation, during training the latent variable is sampled from the approximate posterior $\mathbf{z} \sim q_\phi(\mathbf{z}\mid y_i,\mathbf{x}_i)$, which enables reconstruction-based learning of the generative model. During inference,  the latent variable is instead sampled from the conditional prior $\mathbf{z} \sim p_\psi(\mathbf{z}\mid \mathbf{x}_i)$.
For efficient deterministic prediction, we often simplify this to predicting the mean expectation $\hat{y}_i = \mathbf{\mu}_\theta(\mathbf{z}, \mathbf{x}_i)$.

\subsubsection{Optimization Objective}
The model is trained by maximizing the Evidence Lower Bound (ELBO) of the log-likelihood, and the objective is formulated as:
\begin{equation}
\label{eq:elbo}
\mathcal{L}_{\text{VAE}}
=
\mathbb{E}_{q_\phi(\mathbf{z} \mid y_i, \mathbf{x}_i)}
\left[
\log p_\theta(y_i \mid \mathbf{z}, \mathbf{x}_i)
\right]
-
\text{KL}\!\left(
q_\phi(\mathbf{z} \mid y_i, \mathbf{x}_i)
\,\|\, 
p_\psi(z \mid \mathbf{x_i})
\right),
\end{equation}
where the first term represents the reconstruction loss by Mean Squared Error for Gaussian likelihoods, and the second term is the KL divergence regularizing the posterior towards a prior distribution $p_\psi(\mathbf{z} \mid \mathbf{x_i})$.

However, in standard VAE, the prior  $p_\psi(z \mid \mathbf{x_i})$ is typically assumed to be a standard isotropic Gaussian, i.e., $p(\mathbf{z} \mid \mathbf{x_i}) \sim \mathcal{N}(\mathbf{0}, \mathbf{I})$. While this assumption simplifies the KL computation, it imposes a strong topological constraint, forcing the latent space to be unimodal and globally shared across all users.
However, a static, user-agnostic Gaussian prior is fundamentally ill-suited for such distributional heterogeneity, often leading to over-smoothed predictions or 'posterior collapse'. To address this, we introduce a Flow-based Personalized Prior, which replaces the static $p(\mathbf{z} \mid \mathbf{x_i})$ by explicitly conditioning the latent density on historical interaction patterns.

\subsection{Flow-based Personalized Prior}
To transcend the expressivity limitations of the isotropic Gaussian prior inherent in standard VAEs, we propose a Flow-based Personalized Prior. This module leverages NFs to construct a complex, multimodal prior distribution $p(\mathbf{z} \mid \mathbf{x}_i)$ explicitly conditioned on historical interaction patterns. By learning a sequence of invertible transformations driven by user and item histories, the model adaptively warps a simple base density into a structured, pattern-specific latent manifold.
\subsubsection{Pattern Encoding}
First, to capture the distributional characteristics of historical interactions, we represent the watch time histories of the user ($\mathcal{H}^u_i$) and the item ($\mathcal{H}^v_i$) via their quantile statistics. Let $\mathbf{q}^u = \{q^u_1, \dots, q^u_L\}$ and $\mathbf{q}^v = \{q^v_1, \dots, q^v_L\}$ denote the sequences of $L$ quantiles extracted from $\mathcal{H}^u_i$ and $\mathcal{H}^v_i$, respectively. 
These quantile-based representations encode the global shape of historical watch-time distributions. Together with the static user and item features, they form the contextual condition used for prediction, which we define as $\mathbf{c} = [\mathbf{q}^u, \mathbf{q}^v]$ as the condition for prediction. These ordered sequences provide a compact yet robust summary of long-term viewing patterns.
We employ a Transformer-based architecture to encode these quantile sequences into dense contextual embeddings. Specifically, the quantiles are projected and augmented with positional encodings to preserve their ordinal structure:
\begin{equation}
\mathbf{E}^u = \text{Transformer}(\text{Proj}(\mathbf{q}^u) + \mathbf{P}),\\
\mathbf{E}^v = \text{Transformer}(\text{Proj}(\mathbf{q}^v) + \mathbf{P}),
\end{equation}
where $\mathbf{P}$ denotes the positional embeddings. The self-attention mechanism aggregates distributional patterns across the input sequence, followed by mean pooling:
$\mathbf{h}^u = \text{MeanPool}(\mathbf{E}^u), \mathbf{h}^v = \text{MeanPool}(\mathbf{E}^v)$.
The final conditioning context $\mathbf{h}_{\text{prior}}$ is obtained by concatenating these representations:
\begin{equation}
\mathbf{h}_{\text{prior}} = [\mathbf{h}^u \,\|\, \mathbf{h}^v] \in \mathbb{R}^{2d}.
\end{equation}

\subsubsection{Manifold Warping via Conditional NFs}
Given the context $\mathbf{h}_{\text{prior}}$, we construct the personalized prior distribution by transforming a simple base distribution $\mathbf{z}^0$ inferred by  $z$ through a sequence of $K$ invertible mappings.
Formally, let $f_k: \mathbb{R}^d \mapsto \mathbb{R}^d$ be a bijective function parameterized by $\psi_k$ and conditioned on $\mathbf{h}_{\text{prior}}$. The transformation chain is defined as:
\begin{equation}
\label{eq:flows}
\mathbf{z}^k = f_k(\mathbf{z}^{k-1}; \mathbf{h}_{\text{prior}}), \quad k = 1, \dots, K,
\end{equation}
where $\mathbf{z}^K $ represents the final latent variable sampled from the complex prior.
In this work, we instantiate $f_k$ using Conditional Planar Flows, which offer an efficient mechanism for non-linear expansion and contraction of the density:
\begin{equation}
\mathbf{z}^k = \mathbf{z}^{k-1} + \mathbf{u}_k(\mathbf{h}_{\text{prior}}) \cdot \tanh \left( \mathbf{w}_k(\mathbf{h}_{\text{prior}})^\top \mathbf{z}^{k-1} + b_k(\mathbf{h}_{\text{prior}}) \right),
\end{equation}
where the parameters $\mathbf{u}_k, \mathbf{w}_k, b_k$ are dynamically generated from $\mathbf{h}_{\text{prior}}$ via  MLPs.
According to the change of variables theorem \cite{rezende2015variational}, the log-density of the transformed latent variable $\mathbf{z}$ under the personalized prior $p_\psi(\mathbf{z}^k \mid \mathbf{c})$ can be computed exactly as:
\begin{equation}
\log p_\psi(\mathbf{z}^k \mid \mathbf{c}) = \log p_0(\mathbf{z}^0) - \sum_{k=1}^K \log \det \left|  \frac{\partial f_k(\mathbf{z}^{k-1}; \mathbf{h}_{\text{prior}})}{\partial \mathbf{z}^{k-1}} \right|.
\label{eq:prior_log_prob}
\end{equation}
This formulation allows the model to adaptively warp the isotropic Gaussian $\mathbf{z}^0$ into a complex, user-item pattern-specific manifold. 
Conditioned on the transformed latent variable $\mathbf{z}^K$ and the input features $\mathbf{x}$, the decoder produces the estimation of watch time via
\begin{equation}
    \hat{y} = \mu_\theta(\mathbf{z}^K, \mathbf{x}),
\end{equation}
thereby enabling distribution-aware continuous regression.





\subsection{Joint Optimization Objectives}

To jointly improve prediction accuracy and distributional fidelity, we design a multi-objective loss. 

\paragraph{Accuracy Anchoring}

In the observation space, we adopt the Huber loss $\mathcal{L}_{\text{Huber}}$$\mathcal{L}_{\text{Huber}}$~\cite{huber1992robust},which serves as an accuracy anchor that directly constrains point estimates and improves regression fidelity:
\begin{equation}
\mathcal{L}_{\text{base}}
= \mathcal{L}_{\text{Huber}}(y, \hat{y}) = \mathcal{L}_{\text{Huber}}(y, \mathbf{\mu}_\theta(\mathbf{z}^K, \mathbf{x})).
\end{equation}

\paragraph{Latent Space Regularization}

In the latent space, we align the approximate posterior $q_\phi(\mathbf{z}^0 \mid y, \mathbf{c})$ with a standard Gaussian distribution via KL divergence:
\begin{equation}
\mathcal{L}_{\text{latent}}
=
\text{KL}\!\left(
q_\phi(\mathbf{z}^0 \mid y, \mathbf{c})
\;\|\;
\mathcal{N}(\mathbf{0}, \mathbf{I})
\right)
\end{equation}
This loss preserves continuity and identifiability in the latent space.

\paragraph{Distribution Alignment}

To align the generated watch-time distribution with historical interaction patterns,
we introduce a Wasserstein distance–based loss on quantile representations. Specifically, we generate the samples $\hat{y}^l = \mu_\theta(\mathbf{z}_l^K, \mathbf{x})$ with latent vectors $\mathbf{z}_l^K$ sampled from the NF-based prior. Using the user and item side quantiles $\mathbf{q}^u$ and $\mathbf{q}^v$ defined above, we align the order statistics of these generated samples via the following joint objective:
\begin{equation}
\label{eq:wloss}
\mathcal{L}_{\text{W-Dist}}
=
\frac{1}{L}
\sum_{l=1}^{L}
\Big(
\lambda
\left|
\hat{y}^{(l)} - q^u_l
\right|
+
(1-\lambda)
\left|
\hat{y}^{(l)} - q^v_l
\right|
\Big),
\end{equation}
where $\lambda$ balances the influence of user and item level distributional alignment.

Finally, the whole training objective is as follows:
\begin{equation}
\mathcal{L}
=  \lambda_1 \mathcal{L}_{\text{base}}
+
\lambda_2 \mathcal{L}_{\text{latent}}
+
\mathcal{L}_{\text{W-Dist}}.
\end{equation}

\subsection{Capturing Multimodality via Generative Modeling}

FlowTime adopts a continuous generative regression paradigm that directly
models conditional distributions in continuous spaces.
Assume that under optimization, $\mu_\theta(\boldsymbol{z})$ aligns with
a finite set of dominant conditional modes
$\{\alpha_k\}_{k=1}^K \subset \mathbb R$ of the true conditional
distribution $p_{\mathrm{data}}(y\mid x)$.
We define the latent region associated with mode $k$ as $\mathcal Z_k
\;=\;
\left\{
\boldsymbol{z}
\;\middle|\;
\|\mu_\theta(\boldsymbol{z})-\alpha_k\|\le \varepsilon
\right\}$.

\begin{proposition}[Latent Space Partitioning]
\label{prop:latent-partition}
Assume that $\mu_\theta(\cdot)$ is $L$-Lipschitz continuous and let
\(
\Delta \triangleq \min_{i\neq j}\|\alpha_i-\alpha_j\|
\)
denote the minimum separation between distinct conditional modes.
If $\Delta > 2\varepsilon$, then the latent regions $\{\mathcal Z_k\}$
are pairwise disjoint and satisfy
\begin{equation}
\mathcal Z_i \cap \mathcal Z_j = \varnothing,
\qquad
\|\boldsymbol{z}_i-\boldsymbol{z}_j\|
\ge
\frac{\Delta-2\varepsilon}{L},
\quad
\forall \boldsymbol{z}_i\in\mathcal Z_i,\
\boldsymbol{z}_j\in\mathcal Z_j,\
i\neq j .
\end{equation}
\begin{equation}
\boldsymbol{z}\in\mathcal Z_k
\ \Longrightarrow\
\|\mu_\theta(\boldsymbol{z})-\alpha_k\|\le \varepsilon ,
\end{equation}
i.e., each latent region decodes to a distinct high-density conditional mode,
thereby avoiding collapse to low-density averages.
\end{proposition}

The formal proof and its probabilistic extension under NF-based sampling
are provided in Appendix~\ref{app:latenttheory}.
\begin{table}[t]
  \centering
  \caption{Statistics of the datasets used in our TimeRec Library.}
  \vspace{-1em}
  \label{tab:datasets}
  \resizebox{\columnwidth}{!}{%
  \begin{tabular}{lcccc}
    \toprule
    \textbf{Dataset} & \textbf{\# Users} & \textbf{\# Items} & \textbf{\# Impressions} & \textbf{Characteristics} \\
    \midrule
    KuaiRand-Pure & 27,285 & 7,583 & 1,186,059 & Unbiased \\
    KuaiRec       & 7,176  & 10,728 & 12,530,806 & High Density \\
    \bottomrule
  \end{tabular}%
  }
  \vspace{-1em}
\end{table}

\section{Experiments}
\label{sec:exp}
This section presents extensive experiments to demonstrate the effectiveness of FlowTime.
Five research questions are explored:

\begin{itemize}[itemindent=0pt, left=4pt]
    \item \textbf{RQ1:} Does FlowTime outperform state-of-the-art baselines in both offline accuracy and online commercial metrics? 
    \item \textbf{RQ2:} How is FlowTime compared to other paradigms in terms of computational efficiency versus predictive accuracy? 
    \item \textbf{RQ3:} How effective is each major module in FlowTime?
    \item \textbf{RQ4:} Does FlowTime mitigate the ``mean-collapse'' issue and capture fine-grained user-item interaction patterns?
    \item \textbf{RQ5:} How sensitive is FlowTime to key hyperparameters?  
\end{itemize}

\begin{table*}[t]
\centering
\caption{Performance comparison among different approaches on KuaiRec, KuaiRand, and the Indust dataset. For comparative clarity, FlowTime is grouped under ``Generative Regression" here.}
\label{tab:watch_time}
\resizebox{1.\textwidth}{!}{%
\begin{tabular}{c|c|cccc|cccc|cc}
\toprule
\multirow{2}{*}{\textbf{Paradigm}} & \multirow{2}{*}{\textbf{Method}}
& \multicolumn{4}{c|}{\textbf{KuaiRec}}
& \multicolumn{4}{c|}{\textbf{KuaiRand}}
& \multicolumn{2}{c}{\textbf{Indust}} \\
\cmidrule(lr){3-6}\cmidrule(lr){7-10}\cmidrule(lr){11-12}
& 
& MAE~$\downarrow$ & XAUC~$\uparrow$ & PDF-U~$\downarrow$ & PDF-I~$\downarrow$
& MAE~$\downarrow$ & XAUC~$\uparrow$ & PDF-U~$\downarrow$ & PDF-I~$\downarrow$
& MAE~$\downarrow$ & XAUC~$\uparrow$ \\
\midrule

\multirow{5}{*}{\makecell{Direct\\Regression}}
& VR (\textit{value regression}) & 3.3973 & 0.5822 & 0.3073 & 0.1232 & 22.1504 & 0.6429 & 0.7386 & 0.7363 & 21.3199 & 0.6041 \\
& D2Q~\cite{d2q} (\textit{KDD'22}) & 3.2696 & 0.6043 & 0.3210 & 0.1798 & 19.4258 & \underline{0.6715} & 0.6697 & 0.5665 & 19.6098 & 0.6156 \\
& D2Co~\cite{d2co} (\textit{Recsys'23}) & 3.2633 & 0.5895 & 0.3055 & 0.1346 & 20.7854 & 0.6547 & 0.7435 & 0.6333 & 20.3233 & 0.6053 \\
& CWM~\cite{cwm} (\textit{KDD'24}) & 3.3532 & 0.5899 & 0.3010 & 0.1521 & 19.6351 & 0.6668 & \underline{0.6289} & 0.5908 & 19.9865 & 0.6085 \\
& EGMN~\cite{egmn} (\textit{Recsys'25}) & \underline{3.1803} & \underline{0.6125} & 0.2646 & 0.0874 & 19.3246 & 0.6682 & 0.6503 & \underline{0.5414} & 18.2354 & 0.6155 \\
& DIFL~\cite{jin2026invariant} (\textit{AAAI'26}) & 3.2420 & 0.6055 & 0.3103 & 0.1544 & 19.3181 & 0.6708 & 0.6529 & 0.5374 & 17.5606 & 0.6188 \\
\midrule

\multirow{4}{*}{\makecell{Ordinal\\Regression}}
& TPM~\cite{tpm} (\textit{KDD'23}) & 3.4584 & 0.5819 & 0.2844 & 0.1238 & 22.5950 & 0.6303 & 0.7567 & 0.7770 & 19.2344 & 0.6055 \\
& CREAD~\cite{cread} (\textit{AAAI'24}) & 3.2290 & 0.6123 & 0.2741 & 0.1053 & 19.8087 & 0.6678 & 0.7028 & 0.6479 & 17.9823 & 0.6134 \\
& PTPM~\cite{PTPM} (\textit{CIKM'25}) & 3.2865 & 0.6033 & 0.2787 & 0.1169 & 20.6584 & 0.6679 & 0.7345 & 0.6993 & 18.4263 & 0.6111 \\
& SWaT~\cite{swat} (\textit{KDD'25}) & 3.3496 & 0.5888 & 0.3308 & 0.1168 & 22.3353 & 0.6515 & 0.7456 & 0.7303 & 20.4522 & 0.6098 \\
\midrule

\multirow{2}{*}{\cellcolor{white}\makecell{Generative\\Regression}}
& GR~\cite{ma2024generative} (\textit{WWW'26}) & 3.1985 & 0.6124 & \underline{0.2644} & \underline{0.0833} & \underline{19.2742} & 0.6682 & 0.6664 & 0.6570 & \underline{17.2503} & \underline{0.6201} \\
& \textbf{FlowTime (Ours)}
& \textbf{3.1588} & \textbf{0.6174} & \textbf{0.2634} & \textbf{0.0823} & \textbf{19.1045} & \textbf{0.6751}
& \textbf{0.5974} & \textbf{0.4861} & \textbf{16.9553}
& \textbf{0.6243} \\
\bottomrule
\end{tabular}%
}
\begin{tablenotes}
\footnotesize
\item[*] The best and second best results are marked in \textbf{bold} and \underline{underline}, respectively.
$\uparrow$ indicates that higher values are better, while $\downarrow$ indicates the opposite.
Each experiment is repeated 5 times and the average is reported.
\end{tablenotes}
\end{table*}

\subsection{The TimeRec Library}
To address the prevalence of inconsistent evaluation and foster reproducible research, we develop TimeRec, the first open-source unified benchmarking library for the watch time prediction domain. TimeRec standardizes the entire experimental pipeline by integrating diverse datasets, rigorous feature engineering, and reproduced SOTA methods to ensure fair and consistent comparisons.

\subsubsection{Datasets}
We evaluate all methods on two public datasets and one industrial dataset, with statistics summarized in Tab.~\ref{tab:datasets}.
\begin{enumerate}[itemindent=0pt, left=4pt]
    \item \textbf{KuaiRand-Pure}~\cite{gao2022kuairand}. Derived from Kuaishou's random dispatching policy, this dataset provides unbiased interaction data free from system-induced selection biases, making it ideal for validating intrinsic user interest modeling. 
    \item \textbf{KuaiRec}~\cite{gao2022kuairec}: 
    A fully observed dataset characterized by high density. The rich user-item interactions facilitate precise evaluation of fine-grained patterns and long-term user habits.
    \item \textbf{Indust}: 
    A large-scale industrial dataset sourced from a short-video platform with over 400 million DAUs and multi-billion impressions each day. It exhibits extreme sparsity and long-tail distributions, reflecting real-world deployment challenges.
\end{enumerate}

\subsubsection{Feature Engineering}
\paragraph{Data Splitting} We implement a rigorous time-based splitting strategy for Indust, which utilizes interaction logs from 5 consecutive days for training and the subsequent day for testing. KuaiRand and KuaiRec strictly adhere to their standard chronological train/test splits provided by the original work~\cite{gao2022kuairand,gao2022kuairec}.
\paragraph{Feature Sets} Following~\cite{2024arXiv240215164Y}, we align feature configurations for public datasets. Both utilize `user\_id' and `onehot\_feat0\textasciitilde17' as user features. Item features comprise `item\_id', `duration', and dataset-specific attributes: `feat0\textasciitilde4' for KuaiRec and `feat0\textasciitilde3' for KuaiRand.

\subsubsection{Baselines}
We reproduce representative SOTA methods for rigorous evaluation, categorized into three established paradigms:
\begin{itemize}[leftmargin=*]
\item \textbf{Direct Regression~\cite{d2co,d2q,cwm,egmn}:} 
MSE-based models, including variants augmented with causal debiasing or multi-Gaussian fitting.
\item \textbf{Ordinal Regression~\cite{tpm,cread,PTPM,swat}}
Methods that discretize continuous time into intervals for classification.
\item \textbf{Discrete Generative Regression~\cite{ma2024generative}:} 
Approaches utilizing language modeling to generate quantized time tokens.
\end{itemize}
Differently, our FlowTime represents a novel (or the fourth) paradigm --- \textbf{Continuous Generative Regression}. 


\begin{table}[t]
\centering
\caption{Performance gain on online A/B testing.}
\label{tab:ab_test}
\vspace{-1em}
\resizebox{0.36\textwidth}{!}{ 
\begin{tabular}{l|c|c|c}
\toprule
\textbf{Metric} & \textbf{Uplift} & \textbf{P-value} & \textbf{95\% CI} \\
\midrule
APP Usage Time & +1.027\% & $0.003$ & $[0.82\%, 1.23\%]$ \\
Video Play Time & +1.044\% & $0.001$ & $[0.91\%, 1.18\%]$ \\
Long-view Rate & +0.892\% & $0.002$ & $[0.65\%, 1.13\%]$ \\
Share Rate & +1.236\% & $0.005$ & $[0.95\%, 1.52\%]$ \\
\bottomrule
\end{tabular}
}
\vspace{-0.5em}
\end{table}

\subsubsection{Evaluation Protocols}
We employ two standard metrics and a novel personalization fidelity metric to comprehensively evaluate predictive performance:
\begin{itemize}[leftmargin=*]
    \item \textbf{Mean Average Error (MAE)}: Quantifies point-wise regression accuracy between predicted values $\{\hat{y_i}\}_{i=1}^N$ and actual values $\{y_i\}_{i=1}^N$, denoted as $\frac{1}{N} \sum_{i=1}^{N} \left| \hat{y}_i - y_i \right|$.
    \item \textbf{XAUC~\cite{d2q}}: For all valid pairs $\mathcal{P} = \{(i,j) \mid y_i > y_j\}$, XAUC is defined as the proportion of correctly ordered predictions, denoted as $\frac{1}{|\mathcal{P}|} \sum_{(i,j) \in \mathcal{P}} \mathbb{I}(\hat{y}_i > \hat{y}_j)$, where $\mathbb{I}(\cdot)$ is the indicator function. A higher XAUC indicates superior ranking capability.
    \item \textbf{Personalized Distributional Fidelity (PDF):}
    To evaluate personalized pattern capture, we introduce PDF, defined as the average Jensen-Shannon Divergence between the histogram-binned probability distributions of ground-truth $P_u$ and predicted $Q_u$ watch time sequences for each user:
    \begin{equation}
    \text{PDF-U} = \mathbb{E}_{u} \left[ \frac{1}{2} \text{KL}(P_u \parallel M_u) + \frac{1}{2} \text{KL}(Q_u \parallel M_u) \right]
    \end{equation}
    where $M_u = \frac{1}{2}(P_u+Q_u)$. We report both User-level (PDF-U) and Item-level (PDF-I) scores.
\end{itemize}

\subsubsection{Implementation Details}
All models are trained for 20 epochs using the Adam optimizer ($\beta_1 = 0.9$, $\beta_2 = 0.999$) with a batch size of 1024 and learning rate of $5e^{-4}$. 
To ensure fair comparison, we adopt the optimal hyperparameters reported in the original papers for all baselines. Furthermore, for shared modules across different models (e.g. feature encoders), we strictly align the number of parameters to ensure that performance gains are derived from method superiority rather than model capacity.

\subsection{Performance Comparison (RQ1)}
\subsubsection{Offline Evaluation}
Table~\ref{tab:watch_time} reports the comparative results across three datasets, where FlowTime consistently establishes a new state-of-the-art across all metrics.
First, compared to Direct Regression baselines, FlowTime effectively mitigates the "mean-collapse" issue inherent in MSE-based optimization. This is evidenced by a substantial reduction in MAE by 0.676\% and an XAUC lift of 0.8\% on the KuaiRec dataset.
Second, against Ordinal and Discrete Generative paradigms, FlowTime achieves superior precision by avoiding quantization errors. This advantage is manifested as a 0.704 MAE reduction on KuaiRand, and simultaneous improvements on Indust (MAE 1.027, XAUC 1.777\%).
Most notably, FlowTime demonstrates superior flexibility and effectiveness over the mixture-density-based EGMN, reducing PDF-U and PDF-I by 0.454\% and 5.835\% on KuaiRec, respectively.
This demonstrates the superior expressivity of our flow-based architecture in modeling complex interaction patterns.

\subsubsection{Online A/B Testing}
To validate real-world efficacy, we deploy FlowTime in the Trending Video Recommendation scenario of a leading platform with over 400 million DAUs. 
Serving 10\% of traffic for six days, FlowTime achieves statistically significant gains across all core metrics, as detailed in Table~\ref{tab:ab_test} with corresponding P-values and 95\% Confidence Intervals (CI).
Specifically, it yields substantial commercial value with a \textbf{1.044\%} uplift in Video Play Time and \textbf{1.027\%} in App Usage Time. 
Crucially, it also significantly boosts interaction metrics (e.g., \textbf{+1.236\%} Share Rate), effectively enhancing recommendation accuracy and user satisfaction.

\begin{figure}[t]
    \centering
    \includegraphics[scale=0.18]{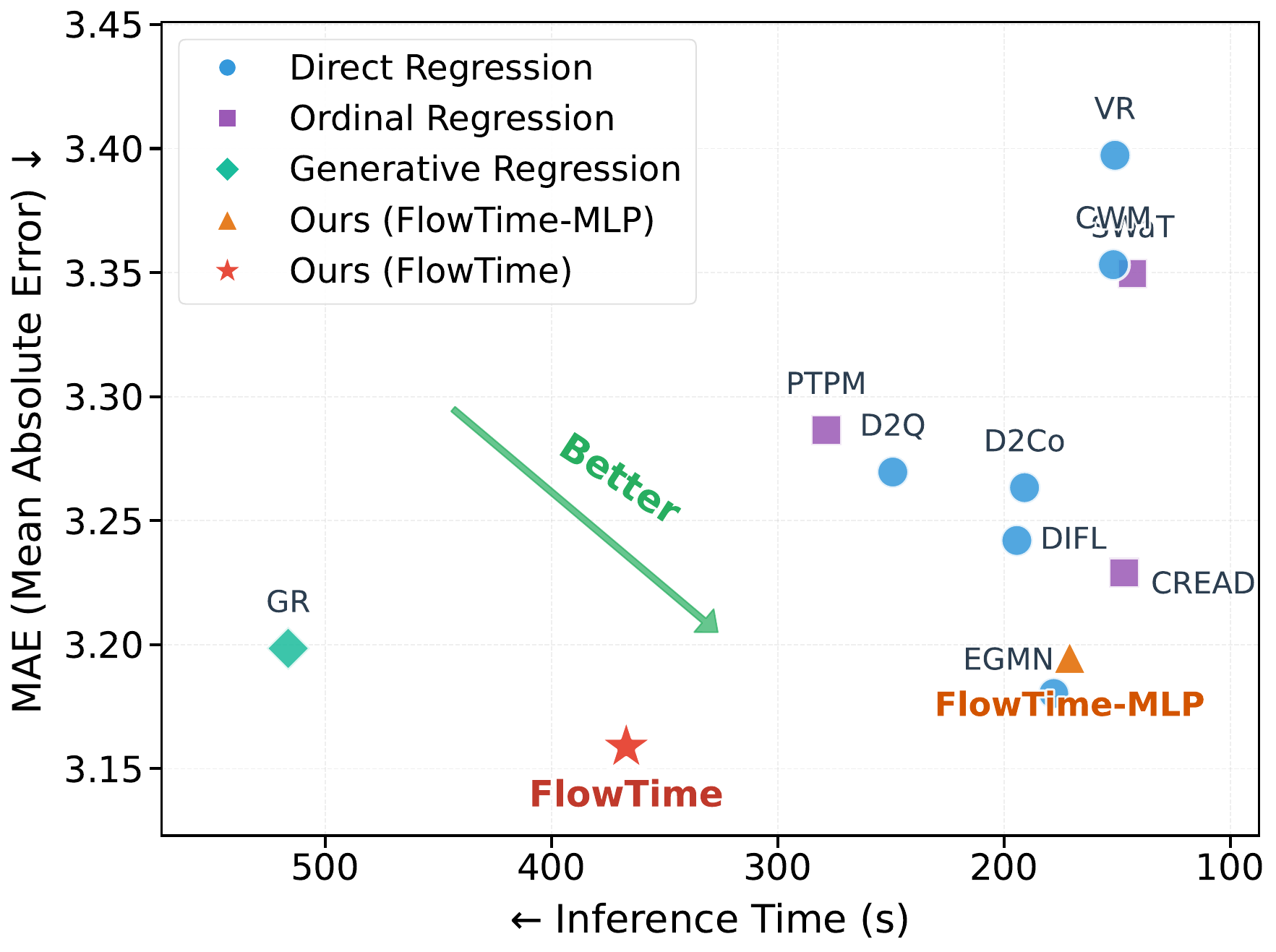}
    \caption{The efficiency-performance trade-off comparison of FlowTime against baselines. The bottom-right corner represents lower error and faster inference.}
    \label{fig:efficiency}
    \vspace{-1em}
\end{figure}

\subsection{Efficiency-Performance Analysis (RQ2)}
\label{sec:efficiency}
Fig.~\ref{fig:efficiency} visualizes the trade-off between inference speed and prediction accuracy across paradigms, where FlowTime achieves a superior balance.
While marginally slower than direct regression models, it offers significantly higher accuracy.
Furthermore, the FlowTime-MLP variant (ablated in Table~\ref{tab:ablation}) further accelerates inference, closing the speed gap.
Crucially, compared to the generative baseline GR, FlowTime realizes a twofold improvement: it enhances accuracy while markedly reducing inference latency by nearly \textbf{30\%}.
This substantiates the dual superiority of our one-step continuous generative architecture and Flow-based Personalized Priors in terms of both efficiency and effectiveness, particularly over iterative autoregressive mechanisms.


\begin{table}[t]
    \centering
    \caption{Ablation study on generative architectures, normalizing flow components, and history modeling strategies.}
    \label{tab:ablation}
    \vspace{-1em}
    \resizebox{0.43\textwidth}{!}{
    \renewcommand{\arraystretch}{0.95}
    \begin{tabular}{l|cccc}
        \toprule
        \textbf{Variant} & \textbf{MAE}  & \textbf{XAUC} & \textbf{PDF-U} & \textbf{PDF-I}\\
        \midrule
        (a) \textbf{FlowTime (Ours)} & \textbf{3.159} & \textbf{0.617} & \textbf{0.249} & \textbf{0.078} \\
        \midrule
        \multicolumn{5}{l}{\textit{Generative Architecture}} \\
        (b) \ w/ DDPM~\cite{ddpm} & 3.163 & 0.613 & 0.251 & 0.081 \\
        (c) \ w/ Flow Matching~\cite{flowmatching} & 3.161 & 0.614 & 0.250 & 0.080 \\
        \midrule
        \multicolumn{5}{l}{\textit{Distribution Modeling}} \\
        (d) \ w/o Normalizing Flows & 3.245 & 0.607 & 0.304 & 0.113 \\
        \midrule
        \multicolumn{5}{l}{\textit{History Modeling}} \\
        (e) \ w/o User History & 3.281 & 0.599 & 0.303 & 0.081 \\
        (f) \ w/o Item History & 3.235 & 0.605 & 0.253 & 0.111 \\
        (g) \ w/o Both \& NF & 3.318 & 0.593 & 0.308 & 0.127 \\
        (h) \ Rep. Trans. w/ MLP & 3.195 & 0.611 & 0.251 & 0.081 \\
        \bottomrule
    \end{tabular}
    }
\end{table}
\subsection{Ablation Study (RQ3)}
We validate the efficacy of FlowTime's core components across three dimensions as shown in Tab.~\ref{tab:ablation}.

\paragraph{Impact of Generative Architecture.}
The mainstream iterative baselines such as Diffusion~\cite{ddpm} and Flow Matching~\cite{flowmatching} outperform SOTA methods, confirming the efficacy of continuous generative modeling.
Nevertheless, FlowTime consistently achieves the best results (Row (a) vs. Rows (b-c)). This superiority suggests that our one-step approach offers a robust inductive bias for high-sparsity scenarios, effectively circumventing the error accumulation inherent in iterative denoising while ensuring low latency.

\paragraph{Necessity of NFs.}
Reverting the Flow-based prior to a standard Gaussian (VAE) degrades MAE by 2.72\% and XAUC by 1.62\% (Row (a) vs. Row (d)). This sharp drop verifies that static priors fail to capture the multimodality of watch time, underscoring the criticality of NFs in modeling complex preference distributions.

\paragraph{Effectiveness of History Modeling.}
(1) Removing both user and item histories (Row (g)) yields the worst performance, confirming the pivotal role of historical interaction patterns in personalization.
(2) User history proves more critical than item history, indicating user habits are the primary drivers (Row (e) vs. Row (f)). (3) Replacing the Transformer with an MLP  (Row (h)) degrades XAUC by $\sim$1\%, highlighting the necessity of attention mechanisms for capturing global behavioral patterns.

\subsection{The Analysis of Personalized Patterns (RQ4)}
\begin{figure}
    \centering
    \includegraphics[scale=0.33]{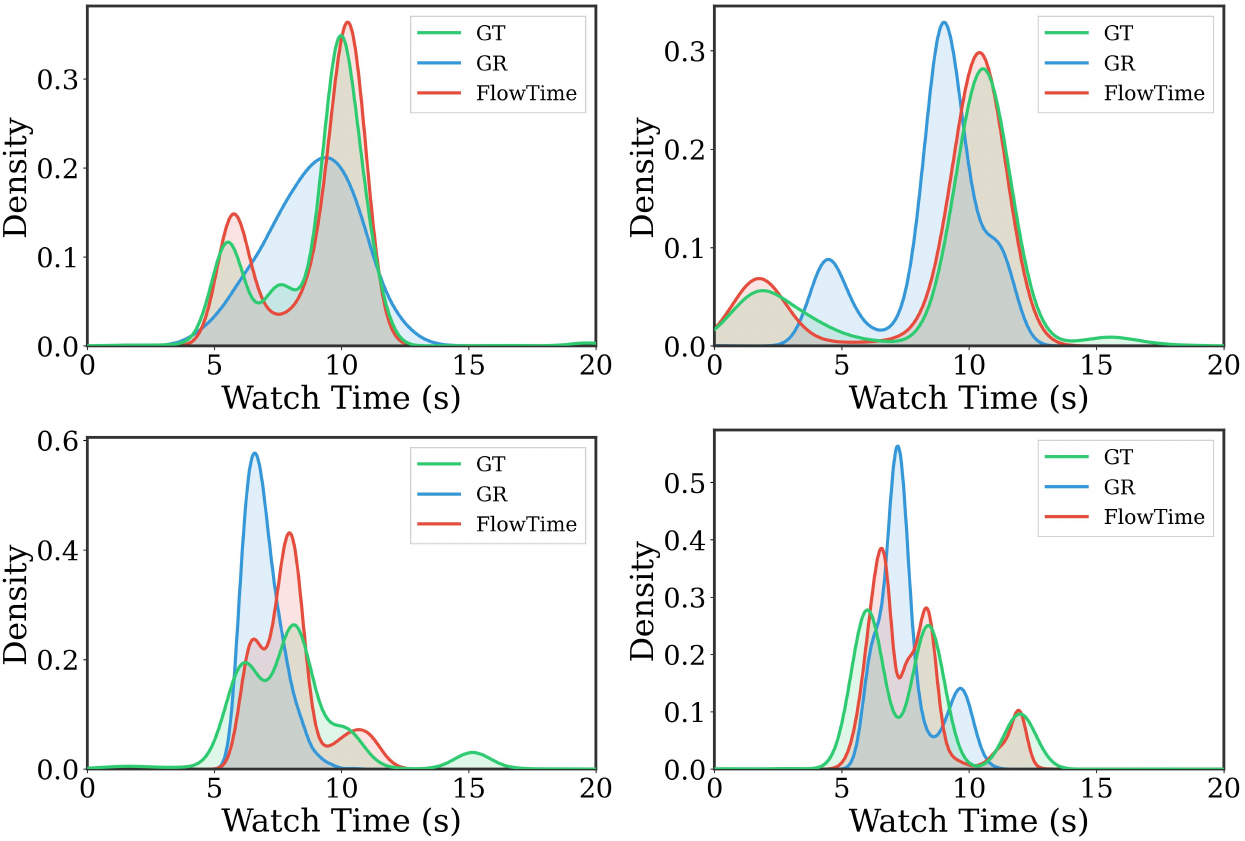}
    \caption{Visualization of distributional fidelity.}
    \label{fig:pred_distribution}
    \vspace{-1em}
\end{figure}
\subsubsection{Prediction Distributions Analysis}
We visualize the predicted watch time distributions for representative user-item pairs to assess the model's capability in capturing complex interaction patterns.
> As shown in Fig.~\ref{fig:pred_distribution}, compared to GR, FlowTime exhibits superior distributional fidelity, accurately reconstructing the Ground Truth (GT) distributions across diverse shapes, including bimodal and trimodal.
This confirms that by modeling the latent manifold via NFs, our framework effectively captures personalized user-item interaction patterns.


\subsection{Effect of Hyperparameters (RQ5)}
\begin{figure}
    \hspace{-15pt}
    \centering
    \includegraphics[scale=0.24]{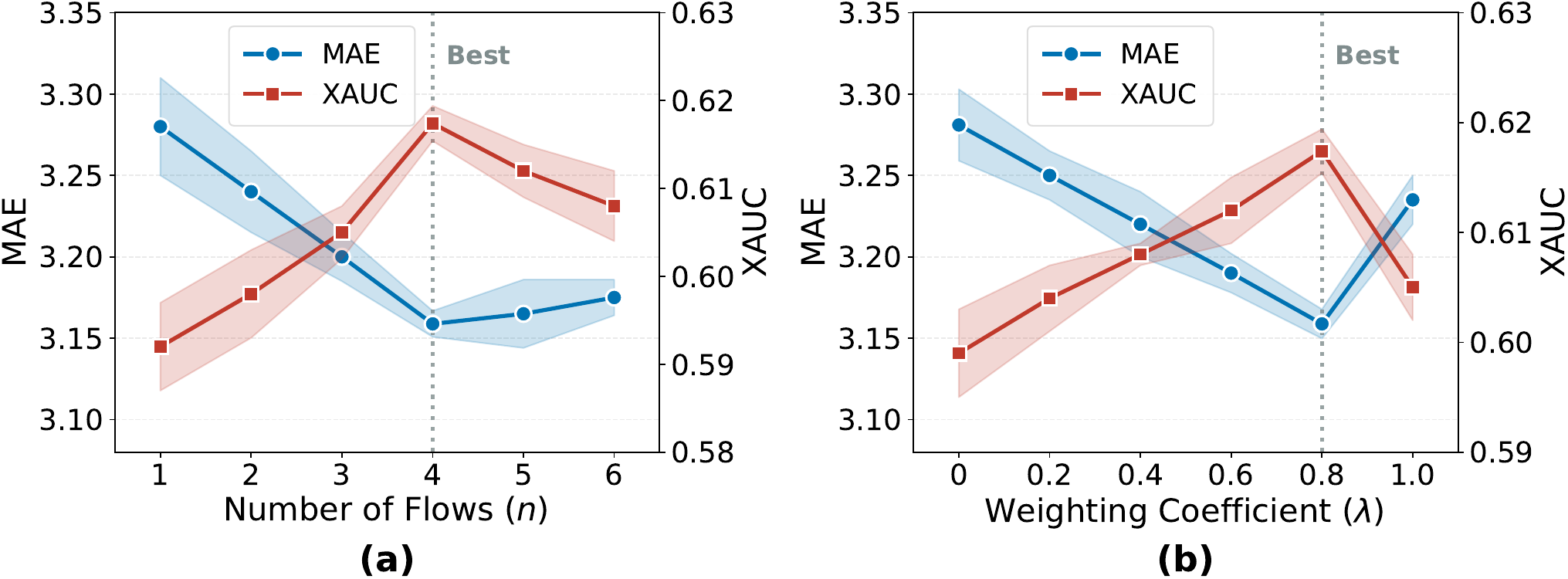}
    \caption{Ablation study on key hyperparameters. (a) Impact of the number of flow steps ($n$) on model performance. (b) Impact of the user-item weighting coefficient ($\lambda$).}
    \label{fig:hyper}
    \vspace{-1em}
\end{figure}
We investigate the impact of two critical hyperparameters: the number of flow steps $K$ (in Eq.~\ref{eq:flows}) and the user-item weighting coefficient $\lambda$ (in Eq.~\ref{eq:wloss}). Fig.~\ref{fig:hyper} illustrates the results on the KuaiRec dataset.
$K$ controls the complexity of the bijective transformation. As shown Fig.~\ref{fig:hyper}(a), performance improves consistently as $K$ increases, peaking at $K=4$, before exhibiting a slight decline.
This suggests that while sufficient flow depth is necessary to model complex manifold topologies, an excessive number of steps may induce overfitting or optimization difficulties.

This coefficient $\lambda$ modulates the trade-off between user and item patterns in constructing the prior latent space.
As $\lambda$ increases, performance steadily improves, corroborating our finding that user-specific patterns are the primary determinants of watch time.
However, performance degrades beyond $\lambda=0.8$.
This inflection point indicates that while user patterns are dominant, item-specific patterns remain indispensable for precise WTP, and neglecting them undermines model fidelity.

\section{Conclusion}
This paper addresses the intrinsic heterogeneity of watch time distributions by proposing \textbf{FlowTime}, a novel framework under the \textit{Continuous Generative Regression} paradigm.
Through the causal modeling of interaction patterns and the design of a Flow-based Personalized Prior, FlowTime effectively captures multimodal user-item interaction patterns while ensuring real-time inference.
Our work not only sets a new SOTA but also bridges the gap in standardized evaluation by constructing the \textit{TimeRec} library, paving the way for more fair and consistent benchmarking in future research.

\noindent \textbf{Acknowledgments}. This work was partially supported by Kuaishou Technology. Jihong Guan was supported by the National Social Science Fund of China (NSFC) under grant No. 24\&ZD185. The computations in this research were performed using the CFFF platform of Fudan University.

\bibliographystyle{ACM-Reference-Format}
\bibliography{main}

\appendix
\section{Theoretical Proofs}
\subsection{Limitations in Ordinal Regression Methods: A Case Study with CREAD}
\label{app:creadtheory}
This section provides a theoretical analysis of these limitations, demonstrating the importance of capturing temporal dependencies to improve prediction accuracy. 
We quantify this error using the Kullback-Leibler (KL) divergence between the true sequential probability distribution and the naive approximation that assumes conditional independence.

Let $\{(x_i, y_i)\}_{i=1}^{N}$ be the training set, where $y_i \in  \mathbb{R}$ is the ground-truth watch time, and $x_i \in \mathbb{R}^d$ denotes the $i$-th input feature vector integrating user- and item-side features.
Discretized ordinal regression methods operate by discretizing the continuous range of $y$ using $M$ thresholds, $c_1 < c_2 < \dots < c_M$, which define a set of intervals.

Instead of modeling the probability of $y_i$ falling into a specific interval, these methods model a sequence of binary decisions. We define a sequence of Bernoulli random variables $\mathbf{B}_i = (\mathbf{B}_i^1, \dots, \mathbf{B}_i^{M})$ for each data point $i$, where:
\begin{equation}
\mathbf{B}_i^m = 1(y_i > c_m)
\end{equation}
Here, $1(\cdot)$ is the indicator function, and $\mathbf{B}_i^m=1$ represents the event that the true value $y_i$ exceeds the threshold $c_m$. The prediction $\hat{y}_i$ is then constructed as a weighted sum of the probabilities of these binary events.

The naive CREAD model implicitly assumes conditional independence across discretized intervals:
\begin{equation}
P_{\text{naive}}(\mathbf{B}_i \mid x_i) = \prod_{m=1}^{M} P(\mathbf{B}_i^m \mid x_i).
\end{equation}

However, these decisions are inherently dependent. The knowledge that $y_i$ has surpassed a threshold $c_{m-1}$ (i.e., $\mathbf{B}_i^{m-1}=1$) provides significant information about whether it will also surpass the next threshold $c_m$. The true conditional distribution factorizes sequentially:
\begin{equation}
P_{\text{true}}(\mathbf{B}_i \mid x_i) = \prod_{m=1}^{M} P(\mathbf{B}_i^m \mid \mathbf{B}_i^{<m}, x_i) ,
\end{equation}
where $\mathbf{B}_i^{<m} = (\mathbf{B}_i^1, \dots, \mathbf{B}_i^{m-1})$ represents the history of binary decisions up to step $m$.

The information loss or error introduced by the naive model's independence assumption can be precisely quantified by the KL divergence between the true and naive distributions:
\begin{equation}
\small
\begin{aligned}
\begin{aligned} D_{KL}(P_{\text{true}} \| P_{\text{naive}}) &= \sum_{\mathbf{B}_i} P_{\text{true}}(\mathbf{B}_i \mid x_i) \log \frac{P_{\text{true}}(\mathbf{B}_i \mid x_i)}{P_{\text{naive}}(\mathbf{B}_i \mid x_i)} \\
&= \sum_{\mathbf{B}_i} P_{\text{true}}(\mathbf{B}_i \mid x_i) \log \frac{\prod_{m=1}^{M} P(\mathbf{B}_i^m \mid \mathbf{B}_i^{<m}, x_i)}{\prod_{m=1}^{M} P(\mathbf{B}_i^m \mid x_i)} \\ &= \sum_{\mathbf{B}_i} P_{\text{true}}(\mathbf{B}_i \mid x_i) \sum_{m=1}^{M} \log \frac{P(\mathbf{B}_i^m \mid \mathbf{B}_i^{<m}, x_i)}{P(\mathbf{B}_i^m \mid x_i)} \\ \end{aligned}
\end{aligned}
\end{equation}

Rearranging the summation terms, we derive the total KL divergence that quantifies the error introduced by ignoring dependencies among discretized intervals:
\begin{equation}
\small
\begin{aligned} D_{KL}(P_{\text{true}} \| P_{\text{naive}}) &= \sum_{m=1}^{M} \sum_{\mathbf{B}_i} P_{\text{true}}(\mathbf{B}_i \mid x_i) \log \frac{P(\mathbf{B}_i^m \mid \mathbf{B}_i^{<m}, x_i)}{P(\mathbf{B}_i^m \mid x_i)} \\ &= \sum_{m=1}^{M} \mathbb{E}_{\mathbf{B}_i \sim P_{\text{true}}} \left[ \log \frac{P(\mathbf{B}_i^m \mid \mathbf{B}_i^{<m}, x_i)}{P(\mathbf{B}_i^m \mid x_i)} \right]  \\
 &= \sum_{m=1}^{M} \mathbb{E}_{\mathbf{B}_i^{<m}}  \left[ D_{KL}^{(m)} \right] \end{aligned}
\end{equation}

This can be explicitly expressed as the conditional KL divergence for each bucket:
\begin{equation}
\begin{aligned}
D_{KL}^{(m)} &= D_{KL}\left(P(\mathbf{B}_i^m \mid x_i, \mathbf{B}_i^{<m})\,\|\,P(\mathbf{B}_i^m \mid x_i)\right) \\
& \sum_{b^m \in \{0,1\}}P(\mathbf{B}_i^m=b^m \mid x_i, \mathbf{B}_i^{<m})\log \frac{P(\mathbf{B}_i^m=b^m \mid x_i, \mathbf{B}_i^{<m})}{P(\mathbf{B}_i^m=b^m \mid x_i)}    
\end{aligned}
\end{equation}
This expectation can be approximated as:
\begin{equation}
    \mathbb{E}_{\mathbf{B}_i^{<m}}\left[D_{K L}^{(m)}\right] \approx I(\mathbf{B}_i^m; \mathbf{B}_i^{<m} \mid x_{i})
\end{equation}
where $I(\mathbf{B}_i^m; \mathbf{B}_i^{<m} \mid x_i)$ denotes the conditional mutual information between the current decision $\mathbf{B}_i^m$ and all preceding decisions $\mathbf{B}_i^{<m}$, given the features $x_i$. The detailed proof is in the following:

\begin{equation}
\small
\begin{aligned} I(\mathbf{B}_i^m; \mathbf{B}_i^{<m} \mid x_i) &\triangleq \sum_{\mathbf{b}^{<m}, b^m} P(\mathbf{b}^{<m}, b^m \mid x_i) \log \frac{P(b^m \mid \mathbf{b}^{<m}, x_i)}{P(b^m \mid x_i)} \\ &= \sum_{\mathbf{b}^{<m}} P(\mathbf{b}^{<m} \mid x_i) \sum_{b^m} P(b^m \mid \mathbf{b}^{<m}, x_i) \log \frac{P(b^m \mid \mathbf{b}^{<m}, x_i)}{P(b^m \mid x_i)} \\ &= \sum_{\mathbf{b}^{<m}} P(\mathbf{b}^{<m} \mid x_i) \left( D_{KL}\left(P(\mathbf{B}_i^m \mid x_i, \mathbf{B}_i^{<m})\,\|\,P(\mathbf{B}_i^m \mid x_i)\right) \right) \\ &\triangleq \mathbb{E}_{\mathbf{B}_i^{<m} \sim P(\cdot \mid x_i)} \left[ D_{KL}^{(m)} \right] \end{aligned}.
\end{equation}

The derived KL divergence decomposition illustrates that the error introduced by the naive discretized modeling approach which ignores dependencies across intervals, can be quantified precisely as the cumulative sum of conditional mutual information across all discretized intervals. Specifically, if intervals are entirely independent (i.e., mutual information $I=0$), the resulting KL divergence error is zero; conversely, if strong dependencies exist among intervals ($I>0$), the error increases proportionally to the strength of these dependencies.

\subsection{Mean Collapse in Regression}
\label{app:regtheory}

\begin{proof}
\textbf{Step 1: Optimal Solution Derivation.} 
The objective is $\min_f \mathbb{E}_{x,y} [\|y - f(x)\|^2]$. Setting the functional derivative to zero:
\begin{equation}
    \frac{\delta}{\delta f} \int \int (y - f(x))^2 p_{\text{data}}(y|x) p(x) dy dx = 0 \implies f^*(x) = \mathbb{E}_{p_{\text{data}}}[y|x].
\end{equation}
Under a Gaussian Mixture Model (GMM) representation of $p_{\text{data}}$ with modes $\mu_k$:
\begin{equation}
    f^*(x) = \sum_{k=1}^K \pi_k(x) \int y \mathcal{N}(y; \mu_k, \sigma^2) dy = \sum_{k=1}^K \pi_k(x) \mu_k(x).
\end{equation}

\textbf{Step 2: Density Analysis in Bimodal Case.} 
Consider $K=2$ with $\pi_1=\pi_2=0.5$. The prediction is the centroid $\bar{\mu} = \frac{\mu_1 + \mu_2}{2}$. The true density at $\bar{\mu}$ is:
\begin{align}
    p_{\text{data}}(\bar{\mu}|x) &= 0.5 \mathcal{N}(\bar{\mu}; \mu_1, \sigma^2) + 0.5 \mathcal{N}(\bar{\mu}; \mu_2, \sigma^2) \\
    &= \frac{1}{\sqrt{2\pi}\sigma} \exp\left(-\frac{(\mu_1 - \mu_2)^2}{8\sigma^2}\right).
\end{align}
Let $\Delta = |\mu_1 - \mu_2|$. Under the disjoint modes assumption ($\Delta \gg \sigma$), the term $\frac{\Delta^2}{8\sigma^2} \to \infty$, thus:
\begin{equation}
    \lim_{\Delta/\sigma \to \infty} p_{\text{data}}(f^*(x)|x) = 0.
\end{equation}
The regression model predicts a value with vanishingly small likelihood, effectively collapsing onto an invalid mean.
\end{proof}

\subsection{Latent Space Partitioning and NF-Driven Generation}
\label{app:latenttheory}

We study a continuous generative regression model with latent variable
$\boldsymbol z\in\mathbb R^d$ and conditional decoder mean
$\mu_\theta(\boldsymbol z,x)$, where $x$ denotes observed conditioning
features and $\{\alpha_k\}_{k=1}^K$ denotes the dominant modes of the
true conditional distribution $p_{\mathrm{data}}(y\mid x)$. For each $k\in\{1,\dots,K\}$, the decoder-induced region are assumed to be pairwise disjoin, defined as:
\begin{equation}
\mathcal Z_k
\triangleq
\left\{
\boldsymbol z:\ \|\mu_\theta(\boldsymbol z,x)-\alpha_k\|\le \varepsilon
\right\}.
\label{eq:def_zk}
\end{equation}

\begin{proposition}[Decoder-Induced Latent Partition]
\label{prop:latent_partition}
Assume $\mu_\theta(\cdot,x)$ is $L$-Lipschitz with respect to $\boldsymbol z$:
\begin{equation}
\|\mu_\theta(\boldsymbol z_1,x)-\mu_\theta(\boldsymbol z_2,x)\|
\le
L\|\boldsymbol z_1-\boldsymbol z_2\|,
\quad \forall \boldsymbol z_1,\boldsymbol z_2.
\label{eq:lipschitz}
\end{equation}

Let $\Delta\triangleq \min_{i\neq j}\|\alpha_i-\alpha_j\|$.
If $\Delta>2\varepsilon$, then $\mathcal Z_i\cap \mathcal Z_j=\varnothing$
for all $i\neq j$. Moreover, for any $\boldsymbol z_i\in\mathcal Z_i$ and
$\boldsymbol z_j\in\mathcal Z_j$ with $i\neq j$,
\begin{equation}
\|\boldsymbol z_i-\boldsymbol z_j\|
\ge
\frac{\Delta-2\varepsilon}{L}.
\label{eq:latent_sep}
\end{equation}
\end{proposition}

\begin{proof}
Fix any $i\neq j$ and take arbitrary
$\boldsymbol z_i\in\mathcal Z_i$ and $\boldsymbol z_j\in\mathcal Z_j$.
By the triangle inequality,
\begin{equation}
\begin{split}
\|\mu_\theta(\boldsymbol z_i,x)-\mu_\theta(\boldsymbol z_j,x)\|
&\ge
\|\alpha_i-\alpha_j\|
-\|\mu_\theta(\boldsymbol z_i,x)-\alpha_i\|
-\|\mu_\theta(\boldsymbol z_j,x)-\alpha_j\| \\
&\ge
\|\alpha_i-\alpha_j\| - 2\varepsilon
\ge
\Delta-2\varepsilon.
\label{eq:decoder_sep}
\end{split}
\end{equation}
By the Lipschitz condition~\eqref{eq:lipschitz}, $\|\boldsymbol z_i-\boldsymbol z_j\|
\ge
\frac{1}{L}
\|\mu_\theta(\boldsymbol z_i,x)-\mu_\theta(\boldsymbol z_j,x)\|
\ge
\frac{\Delta-2\varepsilon}{L}$,
which proves~\eqref{eq:latent_sep}.
If $\Delta>2\varepsilon$ and $\mathcal Z_i\cap \mathcal Z_j\neq\varnothing$,
then there exists $\boldsymbol z\in\mathcal Z_i\cap\mathcal Z_j$,
implying
$\|\mu_\theta(\boldsymbol z,x)-\alpha_i\|\le\varepsilon$ and
$\|\mu_\theta(\boldsymbol z,x)-\alpha_j\|\le\varepsilon$.
Thus $\|\alpha_i-\alpha_j\|\le 2\varepsilon$, contradicting
$\Delta>2\varepsilon$. Hence $\mathcal Z_i\cap \mathcal Z_j=\varnothing$.
\end{proof}

\begin{proposition}[Coverage Implies Mode-Consistent Decoding]
\label{prop:coverage_implies}
Let $p(\boldsymbol z\mid c)$ be any conditional sampling distribution and
define its coverage over decoder-induced regions by
\begin{equation}
\alpha(c)
\triangleq
\mathbb P_{p(\cdot\mid c)}
\!\left(
\boldsymbol z\in\bigcup_{k=1}^K \mathcal Z_k
\right)
=
\sum_{k=1}^K
\mathbb P_{p(\cdot\mid c)}(\boldsymbol z\in\mathcal Z_k).
\label{eq:coverage}
\end{equation}
Define the decoding-success event
 \begin{equation}
 \label{eq:eventE}
E \triangleq
\left\{
\min_{1\le k\le K}
\|\mu_\theta(\boldsymbol z,x)-\alpha_k\|
\le \varepsilon
\right\},
\qquad
\mathbb{P}(E \mid c) \ge \alpha(c).
\end{equation}
\end{proposition}

\begin{proof}
If $\boldsymbol z\in\bigcup_{k=1}^K\mathcal Z_k$, then by the definition of union
there exists some $k\in\{1,\dots,K\}$ such that $\boldsymbol z\in\mathcal Z_k$.
By equation~\ref{eq:def_zk}, $\boldsymbol z\in\mathcal Z_k$ implies $\|\mu_\theta(\boldsymbol z,x)-\alpha_k\|\le\varepsilon$.
Therefore, $\min_{1\le j\le K}\|\mu_\theta(\boldsymbol z,x)-\alpha_j\|
\le
\|\mu_\theta(\boldsymbol z,x)-\alpha_k\|
\le \varepsilon$,
which means $\boldsymbol z\in E$ by the definition of $E$ in
\eqref{eq:eventE}. Hence,  $\left\{\boldsymbol z\in\bigcup_{k=1}^K\mathcal Z_k\right\}\subseteq E$.
Taking probability under $p(\boldsymbol z\mid c)$ and using the monotonicity
 of probability measures yields
\begin{equation}
    \mathbb P(E\mid c)\ge
\mathbb P\!\left(\boldsymbol z\in\bigcup_{k=1}^K\mathcal Z_k \,\middle|\, c\right)
=
\alpha(c).
\end{equation}
\end{proof}

\subsubsection{NF Pullback Mass and Partition Selection}
\label{app:nf_pullback_formal}

\begin{assumption}[NF Mass Concentration]
\label{ass:nf_concentration}
Let $\epsilon\sim\mathcal N(0,I)$ and
$\boldsymbol z=f_\psi(\epsilon;c)$ be an invertible Normalizing Flow
parameterized by $\psi$, conditioned on context $c$.
Let $\delta\in(0,1)$ be a tolerance parameter.
Assume that for the given context $c$ there exists an index
$k(c)\in\{1,\dots,K\}$ such that
\begin{equation}
\int_{f_\psi^{-1}(\mathcal Z_{k(c)};c)}
\mathcal N(\epsilon;0,I)\,d\epsilon
\;\ge\;
1-\delta.
\label{eq:ass_nf_concentration}
\end{equation}
\end{assumption}

\begin{proposition}[NF Coverage under Assumption~\ref{ass:nf_concentration}]
\label{prop:nf_pullback}
Let $p_\psi(\boldsymbol z\mid c)$ denote the conditional distribution
induced by the Normalizing Flow $f_\psi(\cdot;c)$.
Define the NF coverage as
\begin{equation}
\alpha_{\mathrm{NF}}(c)
\triangleq
\mathbb P_{p_\psi(\cdot\mid c)}
\!\left(
\boldsymbol z\in\bigcup_{k=1}^K \mathcal Z_k
\right)
=
\sum_{k=1}^K
\mathbb P_{p_\psi(\cdot\mid c)}(\boldsymbol z\in\mathcal Z_k).
\label{eq:def_alpha_nf}
\end{equation}
Under Assumption~\ref{ass:nf_concentration},
the coverage satisfies
\begin{equation}
\alpha_{\mathrm{NF}}(c)\;\ge\;1-\delta.
\end{equation}
\end{proposition}

\begin{proof}
Since $\boldsymbol z=f_\psi(\epsilon;c)$ is an invertible measurable map,
for any measurable set $S\subseteq\mathbb R^d$ we have the change-of-variables
identity
\begin{equation}
\mathbb P_{p_\psi(\cdot\mid c)}(\boldsymbol z\in S)
=
\mathbb P(\epsilon\in f_\psi^{-1}(S;c))
=
\int_{f_\psi^{-1}(S;c)}
\mathcal N(\epsilon;0,I)\,d\epsilon.
\label{eq:pullback_general}
\end{equation}
In particular, for each decoder-induced region $\mathcal Z_k$,
\begin{equation}
\mathbb P_{p_\psi(\cdot\mid c)}(\boldsymbol z\in \mathcal Z_k)
=
\int_{f_\psi^{-1}(\mathcal Z_k;c)}
\mathcal N(\epsilon;0,I)\,d\epsilon.
\label{eq:pullback_zk}
\end{equation}

By definition~\eqref{eq:def_alpha_nf} and non-negativity of probabilities, $\alpha_{\mathrm{NF}}(c)
= \sum_{k=1}^K
\mathbb P_{p_\psi(\cdot\mid c)}(\boldsymbol z\in\mathcal Z_k)
\;\ge\;
\mathbb P_{p_\psi(\cdot\mid c)}(\boldsymbol z\in\mathcal Z_{k(c)})$.
Applying~\eqref{eq:pullback_zk} with $k=k(c)$ and using
Assumption~\ref{ass:nf_concentration} yields
\begin{equation}
\alpha_{\mathrm{NF}}(c)
\;\ge\;
\int_{f_\psi^{-1}(\mathcal Z_{k(c)};c)}
\mathcal N(\epsilon;0,I)\,d\epsilon
\;\ge\;
1-\delta.
\end{equation}
\end{proof}

We introduce an explicit \emph{conditional assumption}: given a
context $\mathbf{c}$, the NF, during inference-time sampling,
concentrates its probability mass (up to a tolerance $\delta$) on a single
latent partition implicitly induced by the decoder. Under this assumption, we establish a rigorous probabilistic guarantee:
the probability that NF sampling falls into a latent region that can be
decoded into a true conditional mode is at least $1-\delta$.
This result formalizes the role of NF in our framework: it acts as a
\emph{geometry-aware sampling mechanism} that, without participating in
training or regularization, reliably steers generation toward
high-density output modes already captured by the decoder.

\begin{proposition}[Generation Consistency under NF Sampling]
\label{cor:gen_consistency}
Assume the conditions of Proposition~\ref{prop:latent_partition}.
Let $\hat{\boldsymbol z}\sim p_\psi(\boldsymbol z\mid c)$ and
$\hat y=\mu_\theta(\hat{\boldsymbol z},x)$.
If Proposition~\ref{prop:nf_pullback} holds with parameter $\delta$,
then
\begin{equation}
\mathbb P\!\left(
\min_{1\le k\le K}\|\hat y-\alpha_k\|\le\varepsilon
\ \middle|\ c
\right)
\ge 1-\delta.
\end{equation}
\end{proposition}

\begin{proof}
By Proposition~\ref{prop:coverage_implies},
\(
\mathbb P(E\mid c)\ge \alpha_{\mathrm{NF}}(c)
\).
By Proposition~\ref{prop:nf_pullback}, $\alpha_{\mathrm{NF}}(c)\ge 1-\delta$.
Combining yields $\mathbb P(E\mid c)\ge 1-\delta$.
Finally, $E$ is exactly the event
$\min_k\|\mu_\theta(\hat{\boldsymbol z},x)-\alpha_k\|\le\varepsilon$,
and $\hat y=\mu_\theta(\hat{\boldsymbol z},x)$ completes the proof.
\end{proof}

This proposition shows that once the decoder induces mode-aligned latent partitions, it suffices for the Normalizing Flow to select any one of these partitions with high probability during sampling; the generated output will then fall into the corresponding high-density mode of the true conditional distribution with the same probability, ensuring consistent and non-averaged generation.

\subsection{Proof of AR Limitations}
\label{app:arlim}

We provide a bias variance decomposition of the expected squared regression error for tokenized AR models following \cite{ma2026gor}.

\paragraph{Setup and Notation.}
Consider a continuous target reconstructed from a sequence of value tokens:
$y \triangleq \sum_{t=1}^{T} \phi(s^{t})$,
where $s^{t}$ denotes the ground-truth token at step $t$ and $\phi(\cdot)$ maps a token to its numeric value.
The AR model predicts a token sequence $\hat{s}^{1:T}$, yielding the prediction
$\hat{y} \triangleq \sum_{t=1}^{T} \phi(\hat{s}^{t})$.

Define the numeric token values
$C^{t} \triangleq \phi(s^{t}), \hat{C}^{t} \triangleq \phi(\hat{s}^{t})$,
and the step-wise prediction error
$\Delta_t \triangleq \hat{C}^{t} - C^{t}$. By assumption, all token values are bounded:
$C^{t}, \hat{C}^{t} \in [w_{\min}, w_{\max}]$, and the step-wise bias satisfies
$\lvert \mathbb{E}[\Delta_t] \rvert \le B, \forall t$.

\paragraph{Error Decomposition.}
The squared regression error can be written as
\begin{equation}
\mathbb{E}\!\left[(\hat{y}-y)^2\right]
=
\mathbb{E}\!\left[
\left(\sum_{t=1}^{T} \Delta_t \right)^2
\right].
\end{equation}
Using the bias--variance decomposition, we obtain
\begin{equation}
\mathbb{E}\!\left[
\left(\sum_{t=1}^{T} \Delta_t \right)^2
\right]
=
\left(\sum_{t=1}^{T} \mathbb{E}[\Delta_t]\right)^2
+
\mathbb{V}\!\left(\sum_{t=1}^{T} \Delta_t\right).
\end{equation}

\paragraph{Bias Term.}
Let $b_t \triangleq \mathbb{E}[\Delta_t]$. Since $|b_t| \le B$, we have
\begin{equation}
\left(\sum_{t=1}^{T} b_t\right)^2
\le
T^2 B^2.
\end{equation}

\paragraph{Variance Term.}
The variance term expands as
\begin{equation}
\mathbb{V}\!\left(\sum_{t=1}^{T} \Delta_t\right)
=
\sum_{t=1}^{T} \mathbb{V}(\Delta_t)
+
\sum_{t \neq t'} \mathrm{Cov}(\Delta_t, \Delta_{t'}).
\end{equation}
Applying the Cauchy--Schwarz inequality yields
\begin{equation}
\sum_{t \neq t'} \mathrm{Cov}(\Delta_t, \Delta_{t'})
\le
\frac{T(T-1)}{2}
\max_t \mathbb{V}(\Delta_t).
\end{equation}

Since $\Delta_t = \hat{C}^{t} - C^{t}$ and both $\hat{C}^{t}$ and $C^{t}$ lie in $[w_{\min}, w_{\max}]$,
Popoviciu’s inequality gives
\begin{equation}
\mathbb{V}(\Delta_t) \le \frac{(w_{\max}-w_{\min})^2}{4}.
\end{equation}
Therefore,
\begin{equation}
\mathbb{V}\!\left(\sum_{t=1}^{T} \Delta_t\right)
\le
T^2 \cdot \frac{(w_{\max}-w_{\min})^2}{4}.
\end{equation}

\paragraph{Final Bound.}
Combining the bias and variance bounds, we obtain
\begin{equation}
\mathbb{E}\!\left[(\hat{y}-y)^2\right]
\le
T^{2} B^{2}
+
T^{2}\frac{(w_{\max}-w_{\min})^{2}}{4}. \qed
\end{equation}

\paragraph{Vocabulary-induced trade-off.}
Both the sequence length $T$ and the step-wise bias bound $B$ are induced by the vocabulary design.
Let $\mathcal{V}$ denote the value-token vocabulary and $\phi(\mathcal{V}) \subset [w_{\min}, w_{\max}]$ its numeric range.
A finer-grained vocabulary yields smaller token magnitudes and typically requires a longer sequence to represent the same target, leading to larger $T$.
Conversely, a coarser vocabulary reduces $T$ but increases discretization error at each step, enlarging the attainable bias bound $B$.
In particular, since $\Delta_t=\phi(\hat{s}^{t})-\phi(s^{t})$ and $\phi(s^{t}),\phi(\hat{s}^{t})\in[w_{\min},w_{\max}]$, we have
$|\Delta_t| \le w_{\max}-w_{\min}
\quad \Rightarrow \quad
|\mathbb{E}[\Delta_t]| \le B \le w_{\max}-w_{\min}$.
Therefore, tokenized AR regression is intrinsically constrained by a vocabulary-dependent trade-off between $T$ and $B$, which directly controls the error bound in Proposition~\ref{pro:ar_error}.

\end{document}